\documentclass[english]{article}

\usepackage[
letterpaper,
textheight=9in,
textwidth=5.5in,
top=1in
]{geometry}

\usepackage[utf8]{inputenc} 
\usepackage[T1]{fontenc}    
\usepackage{url}            
\usepackage{booktabs}       
\usepackage{amsfonts}       
\usepackage{nicefrac}       
\usepackage{microtype}      

\usepackage{float}
\usepackage{multirow}
\usepackage{amsmath}

\usepackage{amsthm}
\usepackage{amssymb}
\usepackage{graphicx}
\usepackage{esint, epstopdf}
\usepackage{enumitem,multicol,algorithm,algorithmic,bbm,subcaption}
\setlist{itemsep=1pt}

\makeatletter

\providecommand{\tabularnewline}{\\}
\floatstyle{ruled}
\newfloat{algorithm}{tbp}{loa}

\theoremstyle{definition}
\newtheorem{problem}{Problem}
\theoremstyle{plain}
\newtheorem{thm}{Theorem}
\theoremstyle{remark}

\theoremstyle{plain}
\theoremstyle{plain}\newtheorem{defn}{Definition}

\usepackage{xspace}

\makeatother

\usepackage{babel}
\global\long\def\tn{\tilde{N}}
\global\long\def\b{\mathrm{B}}
\global\long\def\d{\mathbf{D}}
\global\long\def\argmax{\operatorname{arg\,max}}
\global\long\def\alg{\textsc{Pulse}\xspace}
\global\long\def\p{\mathrm{Pr}}
\global\long\def\td{\tilde{d}}
\global\long\def\t{\tau}
\global\long\def\av{\mathrm{Avail}}
\global\long\def\e{\mathbb{E}}
\global\long\def\given{\middle\vert}

\newcommand{\er}{Erd\H{o}s-R\'enyi }

\title{Estimating the Size of a Large Network and its Communities from a Random Sample}

\author{
	Lin Chen \\
	Department of Electrical Engineering\\
	Yale University\\
	\texttt{lin.chen@yale.edu} \\
	\and Amin Karbasi\\
	Department of Electrical Engineering\\
	Yale University\\
	\texttt{amin.karbasi@yale.edu} \\
	\and Forrest W.~Crawford\\
	Department of Biostatistics\\
	Yale University\\
	\texttt{forrest.crawford@yale.edu} 
}
\begin{document}

\maketitle

\begin{abstract}
	Most real-world networks are too large to be measured or studied directly and there is substantial interest in estimating global network properties from smaller sub-samples.  One of the most important global properties is the number of vertices/nodes in the network. Estimating the number of vertices in a large network is a major challenge in computer science, epidemiology, demography, and intelligence analysis.  In this paper we consider a population random graph $G=(V,E)$ from the stochastic block model (SBM) with $K$ communities/blocks.  A sample is obtained by randomly choosing a subset $W\subseteq V$ and letting $G(W)$ be the induced subgraph in $G$ of the vertices in $W$. In addition to $G(W)$, we observe the total degree of each sampled vertex and its block membership.  Given this partial information, we propose an efficient  PopULation Size Estimation algorithm,  called \alg, that correctly estimates the size of the whole population as well as the size of each community. To  support our theoretical analysis, we perform an exhaustive set of experiments to study the effects of sample size, $K$, and SBM model parameters on the accuracy of the estimates. The experimental results also demonstrate that \alg significantly outperforms a widely-used method called the network scale-up estimator in a wide variety of scenarios.  We conclude with extensions and directions for future work.
	
	%
	%
	%
	%
\end{abstract}


\section{Introduction}
Many real-world networks cannot be studied directly because they are obscured in some way, are too large, or are too difficult to measure. There is therefore a great deal of interest in estimating properties of large networks via sub-samples \cite{maiya2011benefits}. One of the most important properties of a large network is the number of vertices it contains. Unfortunately census-like enumeration of all the vertices in a network is often impossible, so researchers must try to learn about the size of real-world networks by sampling smaller components.  In addition to the size of the total network, there is great interest in estimating the size of different \emph{communities} or sub-groups from a sample of a network.  Many real-world networks exhibit community structure, where nodes in the same community have denser connections than those in different communities~\cite{girvan2002community,newman2004finding,newman2006modularity}.  In the following examples, we describe network size estimation problems in which only  a small subgraph of a larger network is observed.

\textbf{Social networks.} The social and economic value of an online social network (e.g. Facebook, Instagram, Twitter) is closely related to the number of users the service has.  When a social networking service does not reveal the true number of users, economists, marketers, shareholders, or other groups may wish to estimate the number of people who use the service based on a sub-sample \cite{bernstein2013quantifying}.  

\textbf{World Wide Web.} Pages on the World-Wide Web can be classified into several categories (e.g. academic, commercial, media, government, etc.).
Pages in the same category tend to have more connections. Computer scientists have developed crawling methods for obtaining a sub-network of web pages, along with their hyperlinks to other unknown pages. Using the crawled sub-network and hyperlinks, they can estimate the number of pages of a certain category~\cite{murray2000sizing,massoulie2006peer, ribeiro2010estimating, katzir2011estimating, papagelis2013sampling}. 

\textbf{Size of the Internet.} The number of computers on the Internet (the size of the Internet) is of great interest to computer scientists. However, it is impractical to access and enumerate all computers on the Internet and only a small sample of computers and the connection situation among them can be accessible~\cite{xing2003measuring}.  

\textbf{Counting terrorists.} Intelligence agencies often target a small number of suspicious or radicalized individuals to learn about their communication network. But agencies typically do not know the number of people in the network.  The number of elements in such a covert network might indicate the size of a terrorist force, and would be of great interest~\cite{crawford2015hidden}.

\textbf{Epidemiology.} Many of the groups at greatest risk for HIV infection (e.g. sex workers, injection drug users, men who have sex with men) are also difficult to survey using conventional methods. 
Since members of these groups cannot be enumerated directly, researchers often trace social links to reveal a network among known subjects. Public health and epdiemiological interventions to mitigate the spread of HIV rely on knowledge of the number of HIV-positive people in the population~\cite{kadushin2006scale,guo2013estimating,salganik2011assessing,shelley1995knows,shelley2006knows,shokoohi2012size,ezoe2012population}.  

\textbf{Counting disaster victims.} After a disaster, it can be challenging to estimate the number of people affected. When logistical challenges prevent all victims from being enumerated, a random sample of individuals may be possible to obtain \cite{bernard1988many,bernard2001estimating}.




In this paper, we propose a novel method called $\alg$ for estimating the number of vertices and the size of individual communities from a random sub-sample of the network. We model the network as an undirected simple graph $G=(V,E)$, and we treat $G$ as a realization from the stochastic blockmodel (SBM), a widely-studied extension of the \er random graph model \cite{renyi1959random} that
accommodates community structures in the network by mapping each vertex into one of $K\geq 1$ disjoint types or communities.  We construct a sample of the network by choosing a sub-sample of vertices $W\subseteq V$ uniformly at random without replacement, and forming the induced subgraph $G(W)$ of $W$ in $G$.  We assume that the block membership and total degree $d(v)$ of each vertex $v\in W$ are observed.  We propose a Bayesian esitmation alogrithm $\alg$ for $N=|V|$, the number of vertices in the network, along with the number of vertices $N_i$ in each block.  We first prove important regularity results for the posterior distribution of $N$.  Then we describe the conditions under which relevant moments of the posterior distribution exist.  
We evaluate the performance of $\alg$ in comparison with the popular ``network scale-up'' method (NSUM) \cite{kadushin2006scale,guo2013estimating,salganik2011assessing,shelley1995knows,shelley2006knows,shokoohi2012size,ezoe2012population,killworth1998estimation}. We show that while NSUM is asymptotically unbiased, it suffers from serious finite-sample bias and large variance.
We show that $\alg$ has superior performance -- in terms of relative error and variance -- over NSUM in a wide variety of model and observation scenarios. Proofs are given in the appendix.



\begin{figure}
	\centering
	\includegraphics[width=0.7\columnwidth]{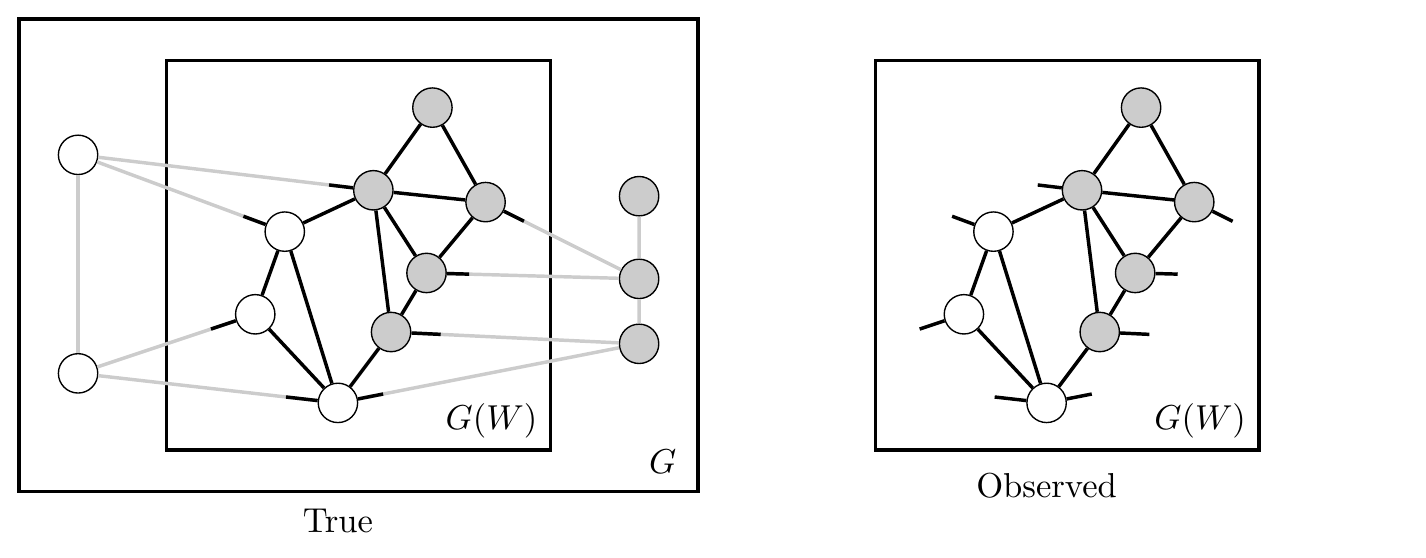}
	\caption{\small Illustration of the vertex set size estimation problem with $N=13$ and $K=2$. 
		White vertices are type-$1$ and gray are type-$2$.
		%
	}
	\label{fig:Problem}
\end{figure}

\section{Problem Formulation}
\vspace{-3mm}
\label{sec:problem-formuation}


The stochastic blockmodel (SBM) is a random graph model that generalizes the \er random graph \cite{renyi1959random}.  Let $G=(V,E)\sim G(N,K,p,t)$ be a realization from an SBM, where $N=|V|$ is the total number of vertices, the vertices  are divided into $K$ types indexed $1,\ldots,K$, specified by the map $t:V\to\{1,\ldots,K\}$, and a type-$i$ vertex and a type-$j$ vertex are connected independently with probability $p_{ij}\in[0,1]$.  Let $N_i$ be the number of type-$i$ vertices in $G$, with $N=\sum_{i=1}^{K}N_{i}$.  The degree of a vertex $v$ is $d(v)$. An edge is said to be of type-$(i,j)$ if it connects a type-$i$ vertex and a type-$j$ vertex.  A random induced subgraph is obtained by sampling a subset $W\subseteq V$ with $|W|=n$ uniformly at random without replacement, and forming the induced subgraph, denoted by $G(W)$. 
Let $V_{i}$ be the number of type-$i$ vertices in the sample and $E_{ij}$ be the number
of type-$(i,j)$ edges in the sample. For a vertex $v$ in the sample,
a \emph{pendant} edge connects vertex $v$ to a vertex outside the sample. 
Let $\tilde{d}(v)=d(v)-\sum_{w\in W}1\{\{w,v\}\in E\}$
be the number of pendant edges incident to $v$. Let $y_{i}(v)$ be
the number of type-$(t(v),i)$ pendant edges of vertex $v$; i.e.,
$y_{i}(v)=\sum_{w\in V\setminus W}1\{t(w)=i,\{w,v\}\in E\}$. We have
$\sum_{i=1}^{K}y_{i}(v)=\tilde{d}(v)$. Let $\tilde{N}_{i}=N_{i}-V_{i}$
be the number of type-$i$ nodes outside the sample. For the ease of presentation we  define $\tn=(\tn_{i}:1\leq i\leq K)$,
$p=(p_{ij}:1\leq i<j\leq K)$, and $y=(y_{i}(v):v\in W,1\leq i\leq K)$.  We observe only $G(W)$ and the total degree $d(v)$ of each vertex $v$ in the sample.
We assume that we know the type of each vertex in the sample.
The observed data $\mathbf{D}$ consists of $G(W)$,
$(d(v):v\in W)$ and $(t(v):v\in W)$; i.e., $\mathbf{D}=(G(W),(d(v):v\in W),(t(v):v\in W))$.
A \textbf{table of notation} is provided in the appendix. 



\begin{problem}
	\label{prob:main}
	Given the observed data $\d$, 
	estimate the size $N$ of the vertex set $N=|V|$ and the size of each community $N_i$.
\end{problem}
Fig. \ref{fig:Problem} illustrates the vertex set size estimation
problem.  White nodes are of type-$1$ and 
gray nodes are of type-$2$. 
All nodes outside $G(W)$ are unobserved.
We observe the types and the total degree of each vertex in the sample.
Thus we know the number of pendant edges that connect each vertex in the sample
to other, unsampled vertices.  However, the destinations of these pendant edges are unknown
to us. 



\section{Network Scale-Up Estimator}
\label{sub:nsu}


We briefly outline a simple and intuitive estimator
for $N=|V|$ that will serve as a comparison to \alg. 
The network scale-up method (NSUM) is a simple 
estimator for the vertex set size of \er
random graphs. It has been used in real-world
applications to estimate the size of hidden or hard-to-reach populations
such as drug users \cite{kadushin2006scale}, HIV-infected individuals
\cite{guo2013estimating,salganik2011assessing,shelley1995knows,shelley2006knows,shokoohi2012size},
men who have sex with men (MSM) \cite{ezoe2012population}, and homeless
people \cite{killworth1998estimation}. Consider a random graph that
follows the \er distribution. 
The expected sum of total degrees in a random sample $W$ of vertices is
$\e\left[\sum_{v\in W}d(v)\right]=n(N-1)p$.
The expected number of edges in the sample is 
$\e\left[E_{S}\right]=\binom{n}{2}p$, where  $E_{S}$ is the number of edges
within the sample. A simple estimator of the connection
probability $p$ is $ \hat{p}=E_{S}/\binom{n}{2}$.  Plugging $\hat{p}$ 
into into the moment equation and solving for $N$ yields 
$\hat{N}=1+(n-1)\sum_{v\in W}d(v)/2E_S$, often simplified to 
$\hat{N}_{NS}=n\sum_{v\in W}d(v)/2E_{S}$ 
\cite{kadushin2006scale,guo2013estimating,salganik2011assessing,shelley1995knows,shelley2006knows,shokoohi2012size,ezoe2012population,killworth1998estimation}. 
\begin{thm}\label{thm:nsu}\textbf{\emph{(Proof in Appendix)}}
	Suppose $G$ follows a stochastic blockmodel with edge probability $p_{ij}>0$
	for $1\leq i,j\leq K$. 
	For any sufficiently large sample size, the
	NSUM is positively biased and $\e[\hat{N}_{NS}|E_S>0]-N$ has an asymptotic lower bound 
	$\e[\hat{N}_{NS}|E_S>0]-N \gtrsim N/n - 1$,
	as $n$ becomes large, where for two sequences $\{a_n\}$ and $\{b_n\}$, $a_n\gtrsim b_n$ means that there exists a sequence $c_n$ such that $a_n\geq c_n\sim b_n$; i.e., $a_n\geq c_n$ for all $n$ and $\lim_{n\to \infty} c_n/b_n=1$. However, as sample size goes to infinity, the NSUM becomes asymptotically unbiased.
\end{thm}



\section{Main Results}

NSUM uses only aggregate information about the sum of the total degrees of vertices in the
sample and the number of edges in the sample. 
We propose a novel 
algorithm, $\alg$, that uses 
individual degree, vertex type, and the network structure information.
Experiments (Section~\ref{sec:Experiment}) show that it outperforms NSUM 
in terms of both bias and variance.

%

\label{sec:main result}






Given $p=(p_{ij}:1\leq i<j\leq K)$, the conditional likelihood of
the edges in the sample is given by 
\[
L_{W}(\mathbf{D};p)=\left(\prod_{1\leq i<j\leq K}p_{ij}^{E_{ij}}(1-p_{ij})^{V_{i}V_{j}-E_{ij}}\right)
\times\left(\prod_{i=1}^{K}p_{ii}^{E_{ii}}(1-p_{ii})^{\binom{V_{i}}{2}-E_{ii}}\right),
\]
and the conditional likelihood of the pendant edges is  given by
\[
L_{\neg W}(\mathbf{D};p)=\prod_{v\in W}\sum_{y(v)}\prod_{i=1}^{K}\binom{\tn_{i}}{y_{i}(v)}p_{i,t(v)}^{y_{i}(v)}(1-p_{i,t(v)})^{\tn_{i}-y_{i}(v)},
\]
where the sum is taken over all $y_{i}(v)$'s ($i=1,2,3,\ldots,K$)
such that $y_{i}(v)\geq0,\forall1\leq i\leq K$ and $\sum_{i=1}^{K}y_{i}(v)=\tilde{d}(v)$.
Thus the total conditional likelihood is 
$
L(\mathbf{D};p)=L_{W}(\d;p)L_{\neg W}(\d;p).
$

If we condition on $p$ and $y$, the likelihood of the edges within
the sample is the same as $L_{W}(\d;p)$ since it does not rely on
$y$, while the likelihood of the pendant edges given $p$ and $y$
is
\[
L_{\neg W}(\d;p,y)=\prod_{v\in W}\prod_{i=1}^{K}\binom{\tn_{i}}{y_{i}(v)}p_{i,t(v)}^{y_{i}(v)}(1-p_{i,t(v)})^{\tn_{i}-y_{i}(v)}.
\]
Therefore the total likelihood conditioned on $p$ and $y$ is given
by 
$
L(\d;p,y)=L_{W}(\d;p)L_{\neg W}(\d;p,y).
$
The conditional likelihood $L(\d;p)$ is indeed a function of $\tn$.
We may view this as the likelihood of $\tilde{N}$ given the data
$\d$ and the probabilities $p$; i.e.,  $L(\tn;\d,p)\triangleq L(\d;p)$.
Similarly, the likelihood $L(\d;p,y)$ conditioned on $p$ and $y$
is a function of $\tn$ and $y$. It can be viewed as the joint likelihood
of $\tn$ and $y$ given the data $\d$ and the probabilities $p$;
i.e.,  $L(\tn,y;\d,p)\triangleq L(\d;p,y)$, and $\sum_{y}L(\tn,y;\d,p)=L(\tn;\d,p)$,
where the sum is taken over all $y_{i}(v)$'s, $v\in W$ and $1\leq i\leq K$,
such that $y_{i}(v)\geq0$ and $\sum_{i=1}^{K}y_{i}(v)=\tilde{d}(v)$,
$\forall v\in W$, $\forall1\leq i\leq K$. To have a full Bayesian approach, we assume that the joint prior distribution
for $\tn$ and $p$ is $\pi(\tn,p)$.  
Hence, the population size estimation problem is equivalent to the following optimization problem for $\tn$:
\begin{equation}\label{eq:opt}
	\hat{\tn}=\argmax\int L(\tn;\d,p)\pi(\tn,p)dp.
\end{equation}
Then we estimate the total population size as $\hat{N}=\sum_{i=1}^{K}\hat{\tn}_{i}+|W|$.



We briefly study the regularity of the posterior distribution of $N$.
In order to learn about $\tilde{N}$, we must observe enough vertices from each block type, and enough edges connecting members of each block, so that the first and second moments of the posterior distribution exist.
Intuitively, in order for the first two moments to exist, either we must observe many edges connecting vertices of each block type, or we must have sufficiently strong prior beliefs about $p_{ij}$.
\begin{thm}
	\label{thm:regularity} \textbf{\emph{(Proof in Appendix)}}     Assume
	that $\pi(\tn,p)=\phi(\tn)\psi(p)$ and 
	$p_{ij}$ follows the Beta distribution
	$\b(\alpha_{ij},\beta_{ij})$ independently for $1\leq i<j\leq K$.
	Let 
	$\lambda=\min_{1\leq i\leq K}\left(\sum_{j=1}^{K}(E_{ij}+\alpha_{ij})\right).$
	If $\phi(\tn)$ is bounded 
	and $\lambda>n+1$, then the $n$-th moment of $N$ exists. 
\end{thm}
In particular, 
if $\lambda>3$, the variance
of $N$ exists. 
Theorem~\ref{thm:regularity} gives the minimum possible number of edges in the sample to make the posterior sampling meaningful. If the prior distribution of $p_{ij}$ is $\mathrm{Uniform}[0,1]$, then we need at least three edges incident on type-$i$ edges for all types $i=1,2,3,\ldots,K$ to guarantee the existence of the posterior variance.

\subsection{\er Model}
In order to better understand how \alg estimates the size of a general stochastic block-model we study the \er case where $K=1$, and all vertices are connected independently with probability $p$.
%
Let $N$ denote the total
population size, $W$ be the sample with size $|W|=V_{1}$ 
and $\tn=N-|W|$.
For each vertex $v\in W$ in the sample, let $\tilde{d}(v)=y(v)$
denote the number of pendant edges of vertex $v$, and
$E=E_{11}$  is the number of edges within the sample. Then
\begin{minipage}{.5\linewidth}
	\[  L_{W}(\d;p)=p^{E}(1-p)^{\binom{|W|}{2}-E},\]
\end{minipage}%
\begin{minipage}{.5\linewidth}
	\[L_{\neg W}(\d;p)=\prod_{v\in W}\binom{\tn}{\tilde{d}(v)}p^{\tilde{d}(v)}(1-p)^{\tn-\tilde{d}(v)}.\]
\end{minipage}
In the \er case, $y(v)=\tilde{d}(v)$ 
and thus $L_{\neg W}(\d;p)=L_{\neg W}(\d;p,y)$. Therefore,
the total likelihood of $\tn$ conditioned on $p$ is given by
\begin{small}
	\[
	L(\tn;\d,p) =  L_{W}(\d;p)L_{\neg W}(\d;p)\\
	=  p^{E}(1-p)^{\binom{|W|}{2}-E}\prod_{v\in W}\binom{\tn}{\tilde{d}(v)}p^{\tilde{d}(v)}(1-p)^{\tn-\tilde{d}(v)}.
	\]
\end{small}
We assume that $p$ has a beta prior $\b(\alpha,\beta)$ and that
$\tn$ has a prior $\phi(\tn)$. 
Let 
\[
L(\tn;\d)=\prod_{v\in W}\binom{\tn}{\tilde{d}(v)}\b(E+u+\alpha,\binom{|W|}{2}-E+|W|\tn-u+\beta),
\]
where $u=\sum_{v\in W}\tilde{d}(v)$. 
The posterior probability $\p[\tn|\d]$ is proportional to $\Lambda(\tn;\d)\triangleq \phi(\tn)L(\tn;\d)$.
The algorithm is presented in Algorithm \ref{alg:erdos}.


\begin{algorithm}
	\algsetup{linenosize=\small}
	\small
	\begin{multicols}{2}
		\begin{algorithmic}[1]
			\REQUIRE Data $\d$; initial guess for $\hat{N}$, denoted
			by $N(0)$; parameters of the beta prior, $\alpha$ and $\beta$
			\ENSURE Estimate for the population size $\hat{N}$
			
			\STATE $\tn(0)\gets N(0)-|W|$
			\STATE $\tau\gets1$  
			
			\REPEAT
			
			\STATE Propose $\tn'(\tau)$ according to a proposal distribution
			$g(\tn(\tau-1)\to\tn'(\tau))$
			
			\STATE $q\gets\min\{1,\frac{\Lambda(\tn'(\tau);\d)g(\tn'(\tau)\to\tn(\tau-1))}{\Lambda(\tn(\tau-1);\d)g(\tn(\tau-1)\to\tn'(\tau))}\}$
			
			

			\STATE $\tn(\tau)\gets\tn'(\tau)$ with probability $q$; otherwise
			$\tn(\tau)\gets\tn(\tau-1)$
			
			\STATE $\tau\gets\tau+1$
			
			\UNTIL{some termination condition is satisfied}
			
			\STATE Look at $\{\tn(\tau):\tau>\tau_{0}\}$ and view it as the sampled
			posterior distribution for $\tn$
			
			\STATE Let $\hat{\tn}$ be the posterior mean with respect to the sampled
			posterior distribution.
			
		\end{algorithmic}
	\end{multicols}
	
	%
	%
	%
	%
	%
	%
	%
	%
	%
	%
	\caption{Population size estimation algorithm $\protect\alg$ (\er
		case)\label{alg:erdos}}
\end{algorithm}

\subsection{General Stochastic Blockmodel Model}

In the \er case, $y(v)=\tilde{d}(v)$.
However, in the general stochastic blockmodel case, in addition to
the unknown variables $\tn_{1},\tn_{2},\ldots,\tn_{K}$ to be estimated,
we do not know $y_{i}(v)$ ($v\in W$, $i=1,2,3,\ldots,K$) either.
The expression $L_{\neg W}(\d;p)$ involves costly summation over all possibilities of integer composition
of $\td(v)$ ($v\in W$).
However, the joint posterior distribution for $\tn$ and $y$, which is
proportional to $\int L(\tn,y;\d,p)\phi(\tn)\psi(p)dp$, does not
involve summing over integer partitions; thus we may sample from the
joint posterior distribution for $\tn$ and $y$, and obtain the marginal
distribution for $\tn$. 
Our proposed algorithm $\alg$ realizes this idea. Let 
$L(\tn,y;\d)=\int L(\tn,y;\d,p)\psi(p)dp$.
We know that the joint posterior distribution for $\tn$ and $y$,
denoted by $\p[\tn,y|\d]$, is proportional to $\Lambda(\tn,y;\d)\triangleq L(\tn,y;\d)\psi(\tn)$.
In addition, the conditional distributions $\p[\tn_{i}|\tn_{\neg i},y]$
and $\p[y(v)|\tn,y(\neg v)]$ are also proportional to $L(\tn,y;\d)\psi(\tn)$,
where $\tn_{\neg i}=(\tn_{j}:1\leq j\leq K,j\neq i)$, $y(v)=(y_{i}(v):1\leq i\leq K)$ and $y(\neg v)=(y(w):w\in W,w\neq v)$.  
The proposed algorithm $\alg$ is a Gibbs sampling process that samples
from the joint posterior distribution (i.e., $\p[\tn,y|\d]$), which is specified
in Algorithm \ref{alg:general}. 

\begin{algorithm}[t!]
	\algsetup{linenosize=\small}
	\small
	\begin{multicols}{2}
		\begin{algorithmic}[1]
			\REQUIRE Data $\d$; initial guess for $\tn$, denoted by $\tn^{(0)}$;
			initial guess for $y$, denoted by $y^{(0)}$; parameters of the beta
			prior, $\alpha_{ij}$ and $\beta_{ij}$, $1\leq i\leq j\leq K$.
			\ENSURE Estimate for the population size $\hat{N}$
			\STATE $\tau\gets 1$
			\REPEAT
			
			\STATE Randomly decide whether to update $\tn$ or $y$
			
			\IF{update $\tn$}
			
			\STATE Randomly selects $i\in[1,K]\cap\mathbb{N}$.
			
			\STATE $\tn^{*}\gets\tn^{(\tau-1)}$
			
			\STATE Propose $\tn_{i}^{*}$ according to the proposal
			distribution $g_{i}(\tn_{i}^{(\tau-1)}\to\tn_{i}^{*})$
			
			
			\STATE $q\gets\min\{1,\frac{\Lambda(\tn^{*},y;\d)g_{i}(\tn_{i}^{*}\to\tn_{i}^{(\tau-1)})}{\Lambda(\tn^{(\tau-1)},y;\d)g_{i}(\tn_{i}^{(\tau-1)}\to\tn_{i}^{*})}\}$
			
			
			\STATE $\tn^{(\tau)}\gets\tn^{*}$ with probability $q$;
			otherwise $\tn^{(\tau)}\gets\tn^{(\tau-1)}$.
			
			\STATE $y^{(\t)}\gets y^{(\t-1)}$
			
			\ELSE
			
			\STATE Randomly selects $v\in W$.
			
			\STATE $y^{*}\gets y^{(\t-1)}$
			
			\STATE Propose $y(v)^{*}$ according to the proposal distribution
			$h_{v}(y(v)^{(\t-1)}\to y(v)^{*})$
			
			\STATE $q\gets\min\{1,\frac{L(\tn,y^{*};\d)h_{v}(y(v)^{*}\to y(v)^{(\t-1)})}{L(\tn,y;\d)h_{v}(y(v)^{(\t-1)}\to y(v)^{*})}\}$
			
			\STATE $y^{(\t)}\gets y^{*}$ with probability $q$; otherwise
			$y^{(\t)}\gets y^{(\t-1)}$.
			
			\STATE $\tn^{(\t)}\gets\tn^{(\t-1)}$
			
			\ENDIF
			
			\STATE $\t\gets\t+1$
			
			\UNTIL{some termination condition is satisfied}
			
			\STATE Look at $\{\tn(\tau):\tau>\tau_{0}\}$ and view it as the sampled
			posterior distribution for $\tn$
			
			\STATE Let $\hat{N}$ be the posterior mean of $\sum_{i=1}^{K}\tn_{i}+|W|$
			with respect to the sampled posterior distribution.%
		\end{algorithmic}
	\end{multicols}

	\caption{Population size estimation algorithm $\protect\alg$ (general stochastic
		blockmodel case)\label{alg:general}}
\end{algorithm}

For every $v\in W$ and $i=1,2,3,\ldots,K$,
$0\leq y_{i}(v)\leq\tn_{i}$ because the number of type-$(i,t(v))$
pendant edges of vertex $v$ must not exceed the total number of type-$i$
vertices outside the sample. Therefore, we have $\tn_{i}\geq\max_{v\in W}y_{i}(v)$
must hold for every $i=1,2,3,\ldots,K$. These observations put constraints
on the choice of proposal distributions $g_{i}$ and $h_{v}$, $i=1,2,3,\ldots,K$
and $v\in W$; i.e., the support of $g_{i}$ must be contained in
$[\max_{v\in W}y_{i}(v),\infty)\cap\mathbb{N}$ and the support of
$h_{v}$ must be contained in 
$
\{y(v):\forall1\leq i\leq K,0\leq y_{i}(v)\leq\tn_{i},\sum_{j=1}^{K}y_{i}(v)=\td(v)\}.
$

Let $\omega_{i}$ be the window size for $\tn_{i}$, taking values in $\mathbb{N}$. Let 
$
l=\max\{\max_{v\in W}y_{i}(v),\tn_{i}^{(\t-1)}-\omega_{i}\}.
$
Let the proposal distribution $g_{i}$ be defined as below:
\[
g_{i}(\tn_{i}^{(\tau-1)}\to\tn_{i}^{*})=\begin{cases}
\frac{1}{2\omega_{i}+1} & \mbox{if }l\leq\tn_{i}^{*}\leq l+2\omega_{i}\\
0 & \mbox{otherwise}.
\end{cases}
\]
The proposed value $\tn_{i}^{*}$ is always
greater than or equal to $\max_{v\in W}y_{i}(v)$. This proposal
distribution uniform within the window $[l,l+2\omega_{i}]$, and
thus the proposal ratio is 
$ g_{i}(\tn_{i}^{*}\to\tn_{i}^{(\tau-1)})/g_{i}(\tn_{i}^{(\tau-1)}\to\tn_{i}^{*})=1$.
The proposal for $y(v)$ and its proposal ratio are presented in the \textbf{appendix}.


\section{Experiment}

\label{sec:Experiment}

\subsection{\er}
%
%

\vspace{-2mm}
\textbf{Effect of Parameter $p$.} We first evaluate the performance of \alg in the \er case.  
We fix the  size of the network at
$N=1000$ and the sample size $|W|=280$ and vary the parameter $p$.
For each $p\in[0.1,0.9]$, we sample
$100$ graphs from  $G(N,p)$. For each selected graph, we compute NSUM and run \alg $50$ times (as it is a randomized algorithm) to compute its performance. We record the relative errors by the Tukey boxplots shown in Fig. \ref{fig:Estimation-error}.
The posterior mean proposed
by $\alg$ is an accurate estimate of the size.
For the parameter $p$ varying from $0.1$ to $0.9$, most of the
relative errors are bounded between $-1\%$ and $1\%$. We also observe that the NSUM tends to overestimate the size as it shows a positive bias. This 
confirms experimentally the result of Theorem \ref{thm:nsu}. For both methods, the interquartile ranges (IQRs, hereinafter) correlate
negatively with $p$. This shows that the variance
of both estimators shrinks when the graph becomes  denser. The
relative errors of  $\alg$ tend to concentrate around
$0$ with larger $p$ which means that the performance of $\alg$ improves with
larger $p$. In contrast, a larger $p$ does not improve
the bias of the NSUM.

\begin{figure}[t!]
	\begin{subfigure}[t!]{.33\textwidth}
		\begin{center}
			\includegraphics[width=\columnwidth]{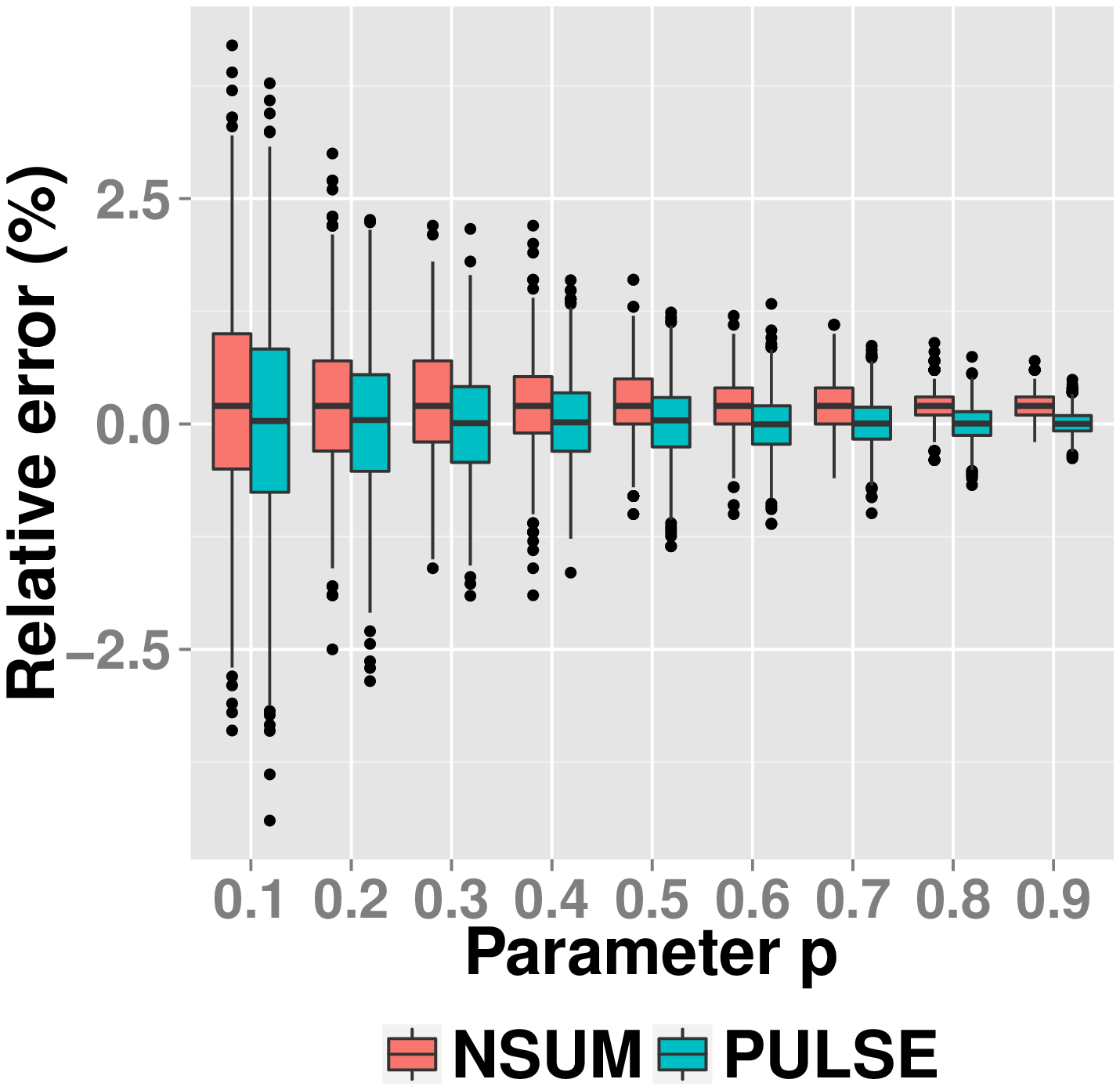}
		\end{center}
		\caption{\label{fig:Estimation-error}}
	\end{subfigure}\hfill
	\begin{subfigure}[t!]{.33\textwidth}
		\begin{center}
			\includegraphics[width=\columnwidth]{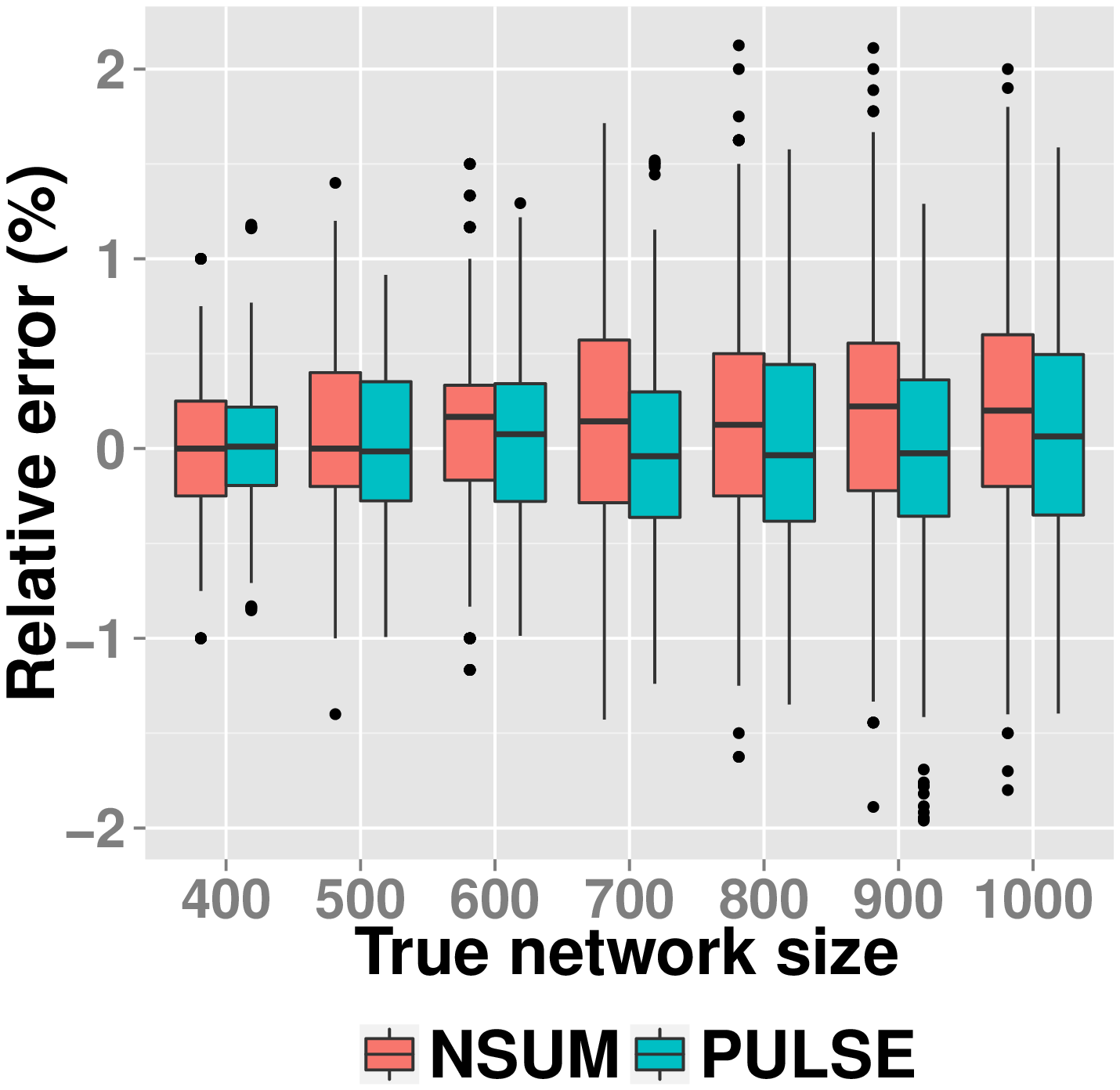}
		\end{center}
		\caption{\label{fig:varyN}}
	\end{subfigure}\hfill
	\begin{subfigure}[t!]{.33\textwidth}
		\begin{center}
			\includegraphics[width=\columnwidth]{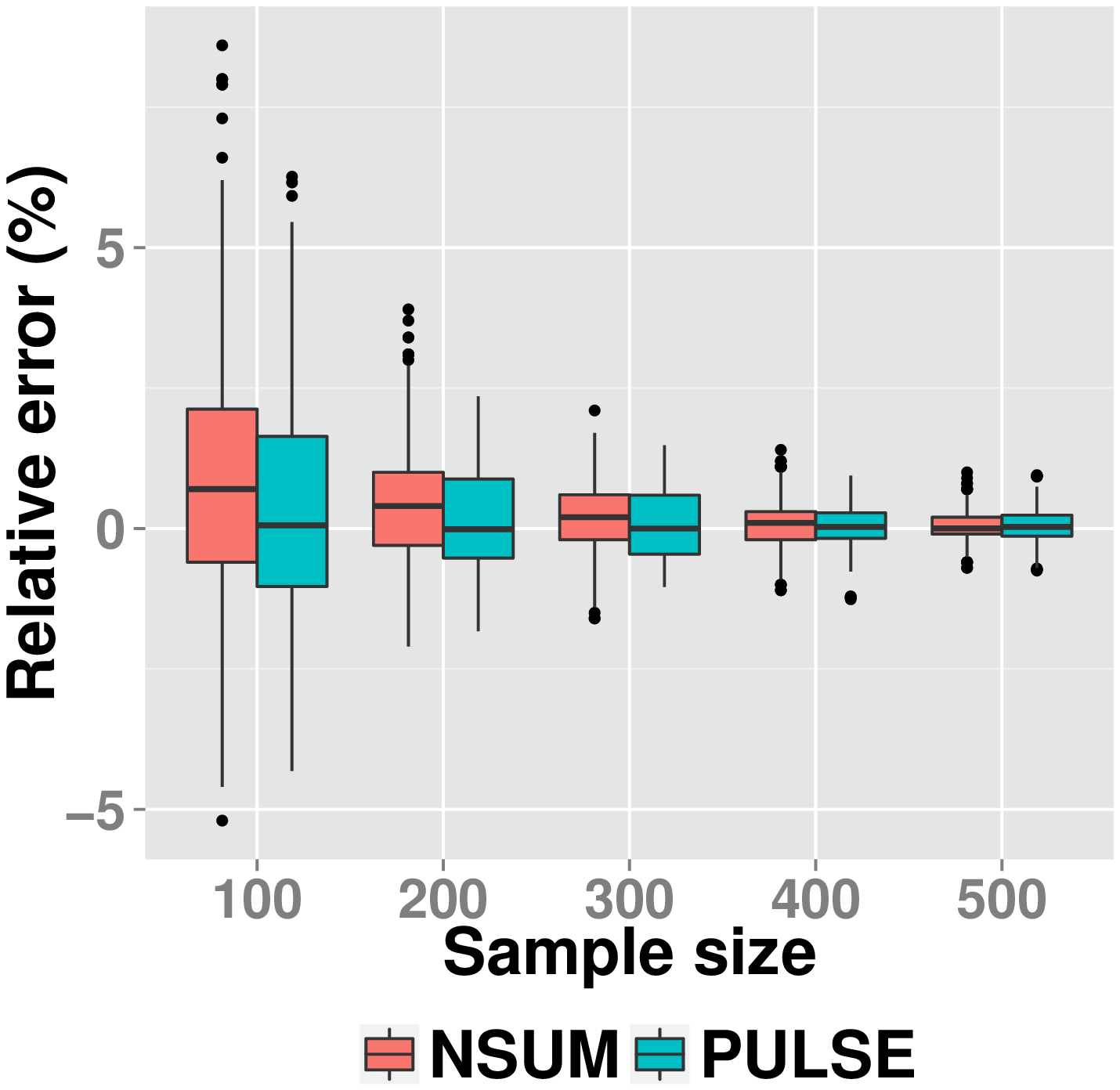}
		\end{center}
		\caption{\label{fig:varyS}}
	\end{subfigure}
	\begin{subfigure}[t!]{.67\textwidth}
		\begin{center}
			\includegraphics[width=\columnwidth]{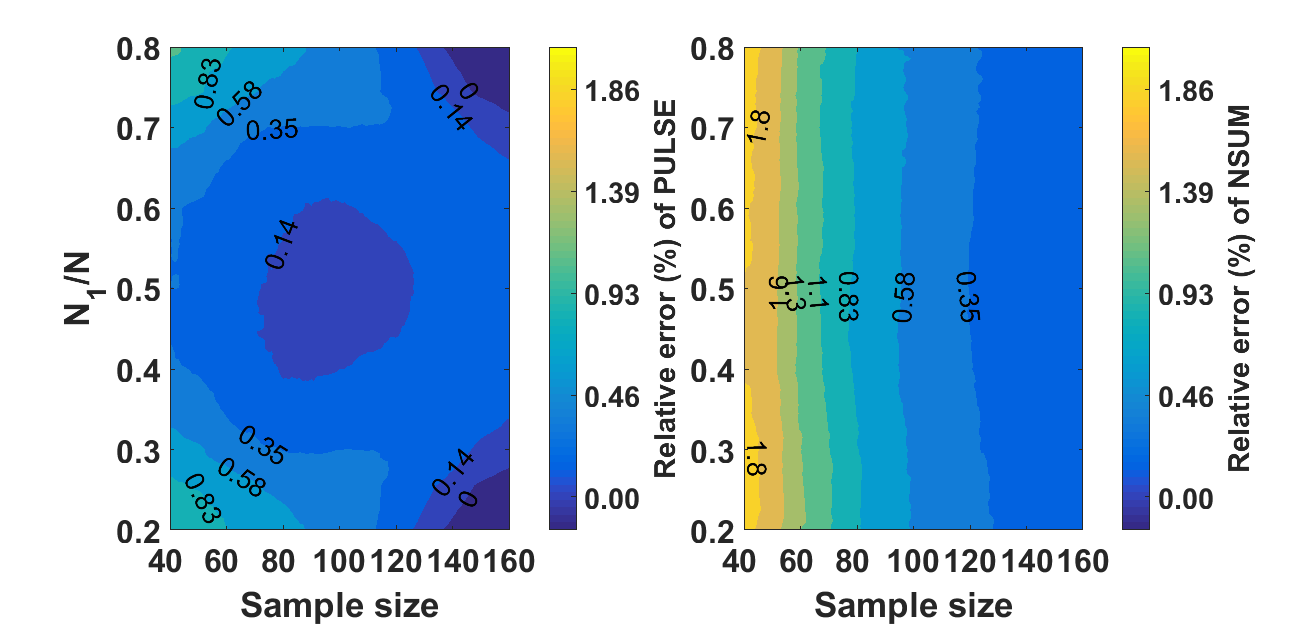}
		\end{center}
		\caption{\label{fig:heatmap}}
	\end{subfigure}
	\begin{subfigure}[t!]{.33\textwidth}
		\begin{center}
			\includegraphics[width=\columnwidth]{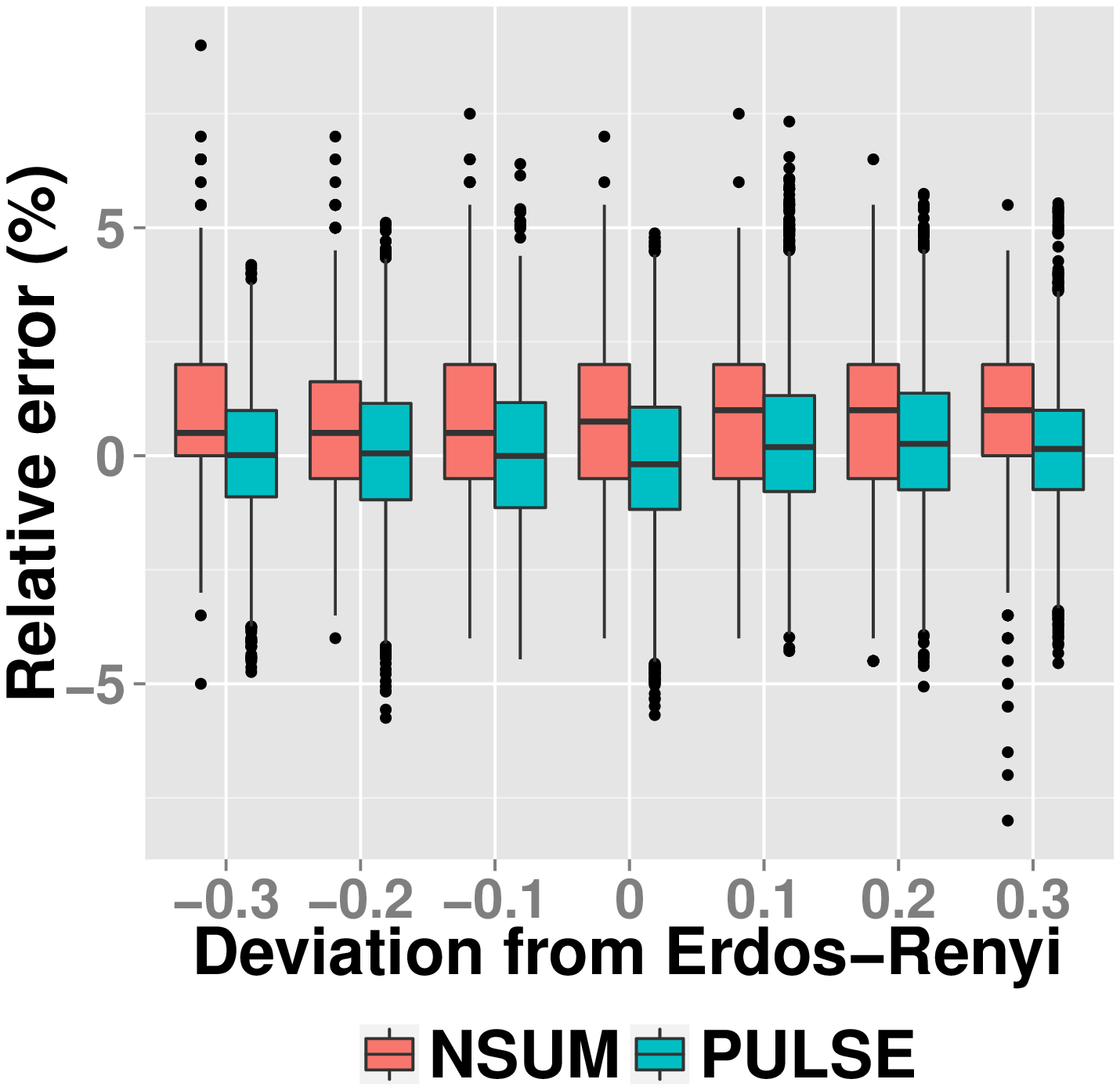}
		\end{center}
		\caption{\label{fig:eps_totalN}}
	\end{subfigure}
	\begin{subfigure}[t]{.33\textwidth}
		\begin{center}
			\includegraphics[width=\columnwidth]{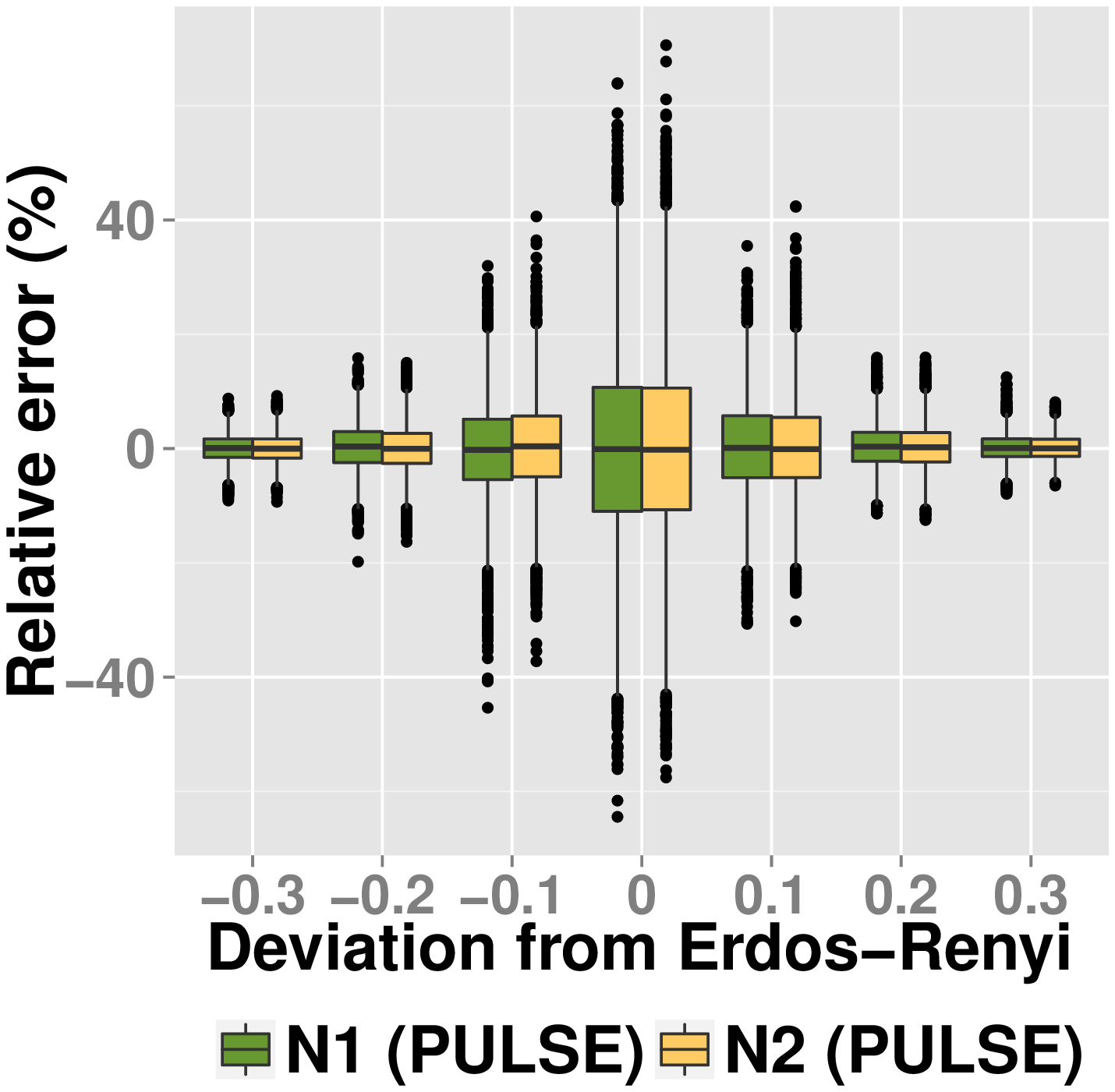}
		\end{center}
		\vspace{-0.4cm}
		\caption{\label{fig:eps_N1}}
	\end{subfigure}\hfill
	\begin{subfigure}[t]{.66\textwidth}
		\centering
		\includegraphics[width=\columnwidth]{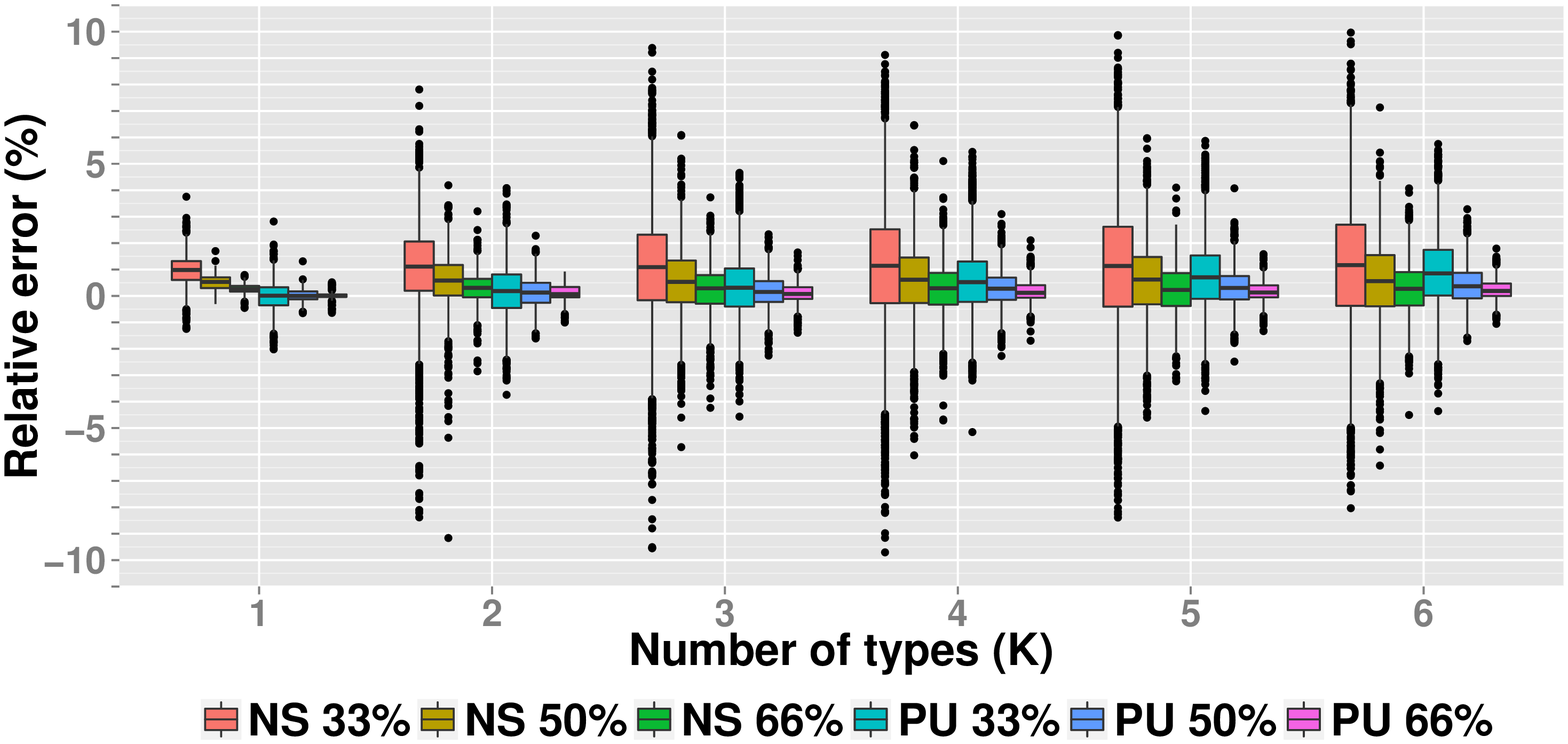}
		\caption{\label{fig:varyK}}
	\end{subfigure}
	\caption{\small Fig.~\ref{fig:Estimation-error}, \ref{fig:varyN} and \ref{fig:varyS} are the results of the \er case: (a) Effect of parameter $p$ on the estimation error. (b) Effect of the  network size on the estimation error. (c) Effect of the sample size on the estimation error. Fig.~\ref{fig:heatmap}, \ref{fig:eps_totalN}, \ref{fig:eps_N1} and \ref{fig:varyK} are the results of the general SBM case: (d) Effect of sample size and type partition on the relative error.  Note that the color bar on the right is on logarithmic scale. (e) Effect of deviation from the \er model (controlled by
		$\epsilon$) on the relative error of NSUM and \alg in the SBM with $K=2$. (f) Effect of deviation from the \er model (controlled by
		$\epsilon$) on the relative  error of   \alg in estimating 
		the number of type-$1$ and type-$2$ nodes in the SBM with $K=2$. (g) Effect of the number of types $K$ and the sample size on the population  estimation. 
		%
		%
		The percentages 
		are the sampling fractions $n/N$. 
		The horizontal axis represents the number of types $K$ that varies
		from $1$ to $6$. The vertical axis is the relative error in percentage.}
\end{figure}

\textbf{Effect of  Network Size $N$.} We fix the parameter $p=0.3$ and the
sample size $|W|=280$ and vary the  network size $N$ from $400$
to $1000$. For each $N\in[400,1000]$, we randomly pick $100$ graphs
from  $G(N,p)$.  For each selected graph, we compute NSUM and run \alg $50$ times. We illustrate the results via Tukey boxplots in Fig.
\ref{fig:varyN}. 
Again, the estimates given by $\alg$ are very accurate. Most
of the relative errors reside in $[-0.5\%,0.5\%]$ and almost
all  reside in $[-1\%,1\%]$. We also observe that  smaller
network sizes can be estimated more accurately as $\alg$ will have  a
smaller variance. For example, when the  network size is $N=400$,
almost all of the relative errors are bounded in the range $[-0.7\%,0.7\%]$
while for  $N=1000$, the relative errors are in  $[-1.5\%,1.5\%]$. This agrees with our intuition that the
performance of estimation improves with a larger sampling fraction. In contrast, NSUM heavily overestimates  the  network size as the size increases. In
addition, its variance also correlates positively with network
size.

\textbf{Effect of Sample Size $|W|$.} We study the effect of the sample size
$|W|$ on the estimation error. Thus, we fix the  size $N=1000$ and the parameter $p=0.3$, and we vary the sample
size $|W|$ from $100$ to $500$. For each $|W|\in[100,500]$, we randomly select $100$ graphs from $G(N,p)$. For every selected graph, we compute the NSUM estimate, run $\alg$ $50$ times, and record the relative
errors. The results are presented in Fig. \ref{fig:varyS}. 
We observe that for both methods that the
IQR shrinks as the sample size increases; thus a larger
sample size reduces the variance of both estimators. 
$\alg$ does not exhibit appreciable bias when the sample size varies from $100$
to $500$. Again, NSUM overestimates the size; however, its bias reduces when the sample
size becomes large. This reconfirms Theorem \ref{thm:nsu}.

\vspace{-2mm}
\subsection{General Stochastic Blockmodel}
\vspace{-1mm}

\textbf{Effect of Sample Size and Type Partition.} Here, we study the effect of the sample size
and the type partition. We set the network size $N$ to $200$ and we assume that there
are two types of vertices in this network: type $1$ and type $2$
with $N_{1}$ and $N_{2}$ nodes, respectively. The ratio $N_{1}/N$
quantifies the type partition. We vary $N_{1}/N$ from $0.2$ to $0.8$
and the sample size $|W|$  from $40$ to $160$. For each combination
of $N_{1}/N$ and the sample size $|W|$, we generate $50$ graphs with $p_{11},p_{22}\sim\mathrm{Uniform}[0.5,1]$
and $p_{12}=p_{21}\sim\mathrm{Uniform}[0,\min\{p_{11},p_{22}\}]$. 
For each graph, we compute the NSUM and obtain the average relative error. Similarly, for each graph, we run $\alg$ $10$ times in order to  compute the average
relative error for the $50$ graphs and $10$ estimates for each graph.
%
The results are shown as  heat maps in Fig.
\ref{fig:heatmap}. Note that the color bar on the right side of Fig. \ref{fig:heatmap}
is on logarithmic scale.
In general, the estimates given by \alg are very accurate and exhibit significant superiority over the NSUM estimates. The largest relative errors of \alg in absolute value, which are approximately $1\%$, appear in the upper-left and lower-left corner on the heat map.  
The performance of the NSUM (see the lower subfigure in Fig.~\ref{fig:heatmap}) is robust to the type partition and equivalently the ratio $N_1/N$. As we enlarge the sample size, its relative error decreases.

The left subfigure in Fig.~\ref{fig:heatmap} shows the performance of \alg. When the sample size is small, the relative error decreases as 
$N_{1}/N$ increases from $0.2$ to $0.5$; when 
$N_{1}/N$
rises from $0.5$ to $0.8$, the relative error becomes large. 
Given the fixed ratio $N_{1}/N$, as expected, the relative error
declines when we have a larger sample. This agrees with our observation
in the \er case. 
However, when the sample size is large, \alg exhibits better performance 
when the type partition is more homogeneous. 
There is a local minimum relative error in absolute value shown at the center of the subfigure. \alg performs best when there is a balance between the number of edges in the sampled induced subgraph and the number of pendant edges emanating outward.  Larger sampled subgraphs allow more precision in knowledge about $p_{ij}$, but more pendant edges allow for better estimation of $y$, and hence each $N_i$. Thus when the sample is approximately half of the total vertex set size, the balanced combination of the number of edges within the sample and those emanating outward leads to better performance. 

\textbf{Effect of Intra- and Inter-Community Edge Probability.} 
Suppose that there are two types of nodes in the network.
The mean degree is given by
\[
d_{\text{mean}}=\frac{2}{N}\left[\binom{N_{1}}{2}p_{11}+\binom{N_{1}}{2}p_{22}+N_{1}N_{2}p_{12}\right].
\]
We want to keep the mean degree constant and vary the random graph
gradually so that we observe 3 phases: high intra-community and low inter-community
edge probability (more cohesive), \er, and low intra-community and high inter-community edge probability
(more incohesive). 
%
%
%
We introduce a cohesion parameter $\epsilon$.
In the two-block model, we have $p_{11}=p_{22}=p_{01}=\tilde{p}$, where $\tilde{p}$ is a constant.
Let's call $\epsilon$ the deviation from this situation and let
$$p_{11}=\tilde{p}+\frac{N_{1}N_{2}\epsilon}{2\binom{N_{1}}{2}}, p_{22}=\tilde{p}+\frac{N_{1}N_{2}\epsilon}{2\binom{N_{2}}{2}},p_{12}=\tilde{p}-\epsilon.$$
%
The mean degree stays constant for different
$\epsilon$. In addition, $p_{11}$, $p_{12}$
and $p_{22}$ must reside in $[0,1]$. This requirement can
be met if we set the absolute value of $\epsilon$ small enough. By changing $\epsilon$ from positive to negative we go from cohesive behavior to incohesive behavior.  Clearly, for $\epsilon=0$, the graph becomes an \er
graph with $p_{11}=p_{22}=p_{01}=\tilde{p}$.

%

We set the network size
$N$ to $850$, $N_1$ to $350$, and $N_2$ to $500$.
We fix $\tilde{p}=0.5$
and let  $\epsilon$ vary from $-0.3$ to $0.3$. When
$\epsilon=0.3$, the intra-community edge probabilities are $p_{11}=0.9298$
and $p_{22}=0.7104$ and the inter-community edge probability is $p_{12}=0.2$.
When $\epsilon=-0.3$, the intra-community edge probabilities are
$p_{11}=0.0702$ and $p_{22}=0.2896$ and the inter-community edge
probability is $p_{12}=0.8$. For each $\epsilon$, we generate $500$
graphs and for each graph, we run $\alg$ $50$ times. Given each
value of $\epsilon$, relative errors are shown in box plots.
%
We present the results  in Fig. \ref{fig:eps_totalN} as we vary $\epsilon$. From Fig. \ref{fig:eps_totalN}, we  observe that despite deviation from the \er graph, both methods are robust. However, the figure indicates that \alg is unbiased (as median is around zero) while NSUM overestimates the size on average. This again confirms Theorem \ref{thm:nsu}.

%

%
An important feature of  \alg is that it can also estimate the number of nodes of each type while NSUM cannot.  The results for  type-$1$
and type-$2$  with different 
$\epsilon$ are shown in Fig.
\ref{fig:eps_N1}. 
We  observe that the median  of
all boxes agree with the $0\%$ line; thus the separate estimates
for $N_{1}$ or $N_{2}$ are unbiased. Note that  when
the edge probabilities are more homogeneous (i.e., when the graph
becomes more similar to the \er model) the IQRs,
as well as the interval between the two ends of the whiskers, become
larger. This shows that when we try to fit an \er model (a
single-type stochastic blockmodel) into a two-type model, the variance
becomes larger.

\vspace{-2mm}
\subsubsection{Effect of Number of Types and Sample Size}

Finally, we study the impact of the number of types $K$
and the sample size
$|W|=n$ on the relative  error. To generate graphs 
with different number of types, we use a Chinese
restaurant process (CRP)~\cite{aldous1985exchangeability}. 
We set the total number of vertices to $200$, first  pick 100 vertices and use the Chinese restaurant process to assign them to different types. Suppose that 
CRP gives
$K$ types; We then distribute the remaining 100 vertices evenly among the $K$ types. 
%
The
edge probability $p_{ii}$ ($1\leq i\leq K$) is sampled from $\mathrm{Uniform}[0.7,1]$
and $p_{ij}$ ($1\leq i<j\leq K$) is sampled from $\mathrm{Uniform}[0,\min\{p_{ii},p_{jj}\}]$,
all independently. We set the sampling fraction $n/N$ to $33\%$, $50\%$ and $66\%$, and use NSUM and \alg to estimate the network size. Relative estimation errors are
illustrated  in Fig. \ref{fig:varyK}. 
We observe that with the same
sampling fraction $n/N$ and same the number of types
$K$,  $\alg$ has a smaller relative error than that of the NSUM. Similarly,  the interquartile
range of \alg is also smaller than that of the NSUM. Hence, \alg provides a higher accuracy with a smaller variance. 
%
%
%
For both methods the relative error decreases (in absolute value) as the sampling fraction
increases. Accordingly, the IQRs also shrink
for larger sampling fraction. 
With the sampling fraction fixed, the IQRs  become larger when we increase the number of
types in the graph.  The variance of both methods 
increases for increasing values of $K$. The median of NSUM  is always above $0$ on average which indicates that it overestimates the network size. 
%


\section{Conclusion}
In this paper, we have developed a method for using a random sub-sample to estimate the size of a graph whose distribution is given by a SBM. We analyzed the bias of the widely-used network
scale-up estimator theoretically and showed that for sufficiently large
sample sizes, it overestimates the vertex
set size in expectation (but asymptotically unbiased).
Regularity results establish the conditions under which the posterior distribution of the population size is well-defined.
Extensive experimental results show that $\alg$
outperforms the network scale-up estimator in terms of the relative
error and estimation variance.

\label{sec:Conclusion}

\newpage
\bibliographystyle{abbrv}
{\small
\bibliography{reference-list}}


\newpage

\section*{Appendix}
\subsection*{Table of Notation}

Notation in this paper is summarized in Table~\ref{tab:Notation}.

\begin{table}[htb]
	\small\centering
	\begin{tabular}{|c|l|c|l|}
		\hline 
		$G=(V,E)$& Underlying graph structure & $E_{ij}$ & \# of type-$(i,j)$ edges in the sample  \tabularnewline
		\hline 
		$N=|V|$ & True population size & $E_{S}$ & \# of edges in the sample\tabularnewline
		\hline 
		$K$  & \# of types & $y_{i}(v)$ & \# of type-$(t(v),i)$ pendant edges of $v$\tabularnewline
		\hline 
		$t(v)$ & Type of vertex $v\in V$ & $\td(v)$ & \# of pendant edges of vertex $v\in W$\tabularnewline
		\hline 
		$d(v)$ & Degree of vertex $v\in V$ in $G$ & $\tn_{i}$ & \# of type-$i$ vertices outside sample\tabularnewline
		\hline 
		$N_{i}$ & Total \# of type-$i$ vertices & $\tn$ & $(\tn_{i}:1\leq i\leq K)$\tabularnewline
		\hline 
		$W$ & Vertex set of the sample
		& $p$ & $(p_{ij}:1\leq i<j\leq K)$\tabularnewline
		\hline 
		$n=|W|$ & Sample size & $y(v)$ & $(y_{i}(v):1\leq i\leq K)$\tabularnewline
		\hline 
		$G(W)$ & Subgraph of $G$ induced by $W$ & $y$ & $(y_{i}(v):v\in W,1\leq i\leq K)$\tabularnewline
		\hline 
		$V_{i}$ & \# of type-$i$ vertices in the sample &&\tabularnewline
		\hline 
	\end{tabular}
	\caption{Notation\label{tab:Notation}}
\end{table}

\subsection*{Proof of Theorem \ref{thm:nsu}}

Define $p$ is the probability that two different nodes $u$ and $v$
have an edge between them when they are sampled uniformly at random
from the vertex set. Thus the probability $p$ is given by
\[
p=\sum_{i=1}^{K}\frac{\binom{N_{i}}{2}}{\binom{N}{2}}p_{ii}+\sum_{1\leq i<j\leq K}\frac{N_{i}N_{j}}{\binom{N}{2}}p_{ij}.
\]
Let's compute the expectation of the network scale-up estimator. We
have
\begin{eqnarray*}
	& &  \e\left[\hat{N}_{NS}|E_{S}>0\right]= \e\left[n\cdot\frac{\sum_{v\in W}d(v)}{2E_{S}}\middle\vert E_{S}>0\right]\\
	& = & n\cdot\e\left[\frac{\sum_{v\in W}\td(v)+2E_{S}}{2E_{S}}\middle\vert E_{S}>0\right]\\
	& = & n\cdot\e\left[\frac{\sum_{v\in W}\td(v)}{2E_{S}}+1\middle\vert E_{S}>0\right]\\
	& = & n\cdot\left(\frac{\e\left[\sum_{v\in W}\td(v)\right]\e\left[\frac{1}{E_{S}}\middle\vert E_{S}>0\right]}{2}+1\right)\\
	& = & n\cdot\left(\frac{n(N-n)p}{2}\e\left[\frac{1}{E_{S}}\middle\vert E_{S}>0\right]+1\right).
\end{eqnarray*}
By Jensen's inequality, we know that $$\e\left[\frac{1}{E_{S}}\given E_{S}>0\right]\geq\frac{1}{\e\left[E_{S}\given E_{S}>0\right]}.$$
Plugging this in, we have
\begin{eqnarray*}
	&  & \e\left[\hat{N}_{NS}|E_{S}>0\right]\geq n\cdot\left(\frac{n(N-n)p}{2\e\left[E_{S}|E_{S}>0\right]}+1\right)\\
	& = & n+(N-n)\cdot\left(1+\frac{1}{n-1}\right)\left(1-\p\left[E_{S}=0\right]\right)
\end{eqnarray*}
where 
\[
\p\left[E_{S}=0\right]=\prod_{1\leq i<j\leq K}(1-p_{ij})^{V_{i}V_{j}}\prod_{i=1}^{K}(1-p_{ii})^{\binom{V_{i}}{2}}.
\]
Note that when 
$
\frac{1}{n}>\p\left[E_{S}=0\right],
$
we have 
$$
\left(1+\frac{1}{n-1}\right)\left(1-\p\left[E_{S}=0\right]\right)>1.
$$
When the sample size $n$ is sufficiently large, the inequality 
$\frac{1}{n}>\p\left[E_{S}=0\right]$
will hold because the terms $(1-p_{ij})^{V_{i}V_{j}}$ and $(1-p_{ii})^{\binom{V_{i}}{2}}$
decrease exponentially. In this case, 
$$
\e\left[\hat{N}_{NS}|E_{S}>0\right]>n+(N-n)=N.
$$
Therefore for a sufficiently large sample size, the network scale-up
estimator is biased and always overestimates the vertex set size. Furthermore, in addition to showing that it always overestimates the vertex set, we can derive an asymptotic lower bound for the bias via a more careful analysis.

Let us recall the definitions of asymptotic equality and inequality for completeness.

\begin{defn}
	Let $\{a_n\}$ and $\{b_n\}$ be two sequences of real numbers. We say that $\{a_n\}$ and $\{b_n\}$ are asymptotically equal if $\lim_{n\to \infty} a_n/b_n=1$; in this case, we denote it by\[a_n\sim b_n,n\to \infty\]
\end{defn}
\begin{defn}
	Let $\{a_n\}$ and $\{b_n\}$ be two sequences of real numbers. We say that $\{a_n\}$ is asymptotically greater than or equal to $\{b_n\}$ if there exists a sequence $\{c_n\}$ such that $a_n\geq c_n$ for all $n$ and $c_n\sim b_n$; in this case, we denote it by\[a_n\gtrsim b_n,n\to \infty.\]
\end{defn} 

Recall that we just showed that\[
\e\left[\hat{N}_{NS}|E_{S}>0\right]\geq   n+(N-n)\cdot\left(1+\frac{1}{n-1}\right)\left(1-\p\left[E_{S}=0\right]\right).
\]
Then we have\[
\e\left[\hat{N}_{NS}|E_{S}>0\right]-N \geq \frac{N-n}{n-1}\left(1-n \cdot\p\left[E_{S}=0\right]\right).
\]
Since $\p\left[E_{S}=0\right]$ decreases to $0$ exponentially in $n$, we have $\lim_{n\to \infty} n\cdot \p\left[E_{S}=0\right] =0$. Thus we know that\[
\frac{N-n}{n-1}\left(1-n \cdot\p\left[E_{S}=0\right]\right)\sim \frac{N-n}{n-1} \sim N/n-1. 
\]
Therefore, we deduce that\[\e\left[\hat{N}_{NS}|E_{S}>0\right]-N\gtrsim N/n-1.\]

Now we would like to show its asymptotic unbiasedness. We have 
\[
\e\left[\hat{N}_{NS}|E_{S}>0\right]=n\cdot\left(\frac{n(N-n)p}{2}\e\left[\frac{1}{E_{S}}\middle\vert E_{S}>0\right]+1\right).
\]
To show the asymptotic unbiasedness, we have to derive an upper bound
for the conditional expectation $\e\left[\frac{1}{E_{S}}\given E_{S}>0\right]$.
Let $\delta>0$ be a constant to be determined later. We divide it
into two cases where $E_{S}$ is concentrated around its mean and
the anti-concentration happens:
\begin{eqnarray}
	\e\left[\frac{1}{E_{S}}\middle\vert E_{S}>0\right]\nonumber & = & \e\left[\frac{1}{E_{S}}\middle\vert E_{S}>0,\left|E_{S}-\binom{n}{2}p\right|\leq\delta\right]\nonumber \\
	& \times & \p\left[\left|E_{S}-\binom{n}{2}p\right|\leq\delta\given E_{S}>0\right]\nonumber \\
	& + & \e\left[\frac{1}{E_{S}}\middle\vert E_{S}>0,\left|E_{S}-\binom{n}{2}p\right|>\delta\right]\nonumber \\
	& \times & \p\left[\left|E_{S}-\binom{n}{2}p\right|>\delta\given E_{S}>0\right].\label{eq:two-terms}
\end{eqnarray}
Given $E_{S}>0$, $1/E_{S}$ is always less than or equal to $1$.
Thus the second term in (\ref{eq:two-terms}) can be bounded as below:
\begin{eqnarray*}
	&  & \e\left[\frac{1}{E_{S}}\middle\vert E_{S}>0,\left|E_{S}-\binom{n}{2}p\right|>\delta\right]\times\\
	&  & \p\left[\left|E_{S}-\binom{n}{2}p\right|>\delta\given E_{S}>0\right]\\
	& \leq & \p\left[\left|E_{S}-\binom{n}{2}p\right|>\delta\given E_{S}>0\right]\\
	& \leq & \frac{\p\left[\left|E_{S}-\binom{n}{2}p\right|>\delta\right]}{\p\left[E_{S}>0\right]}\\
	& \leq & \frac{2e^{-2\delta^{2}/\binom{n}{2}}}{1-\p\left[E_{S}=0\right]},
\end{eqnarray*}
where the last inequality is a result of Hoeffding's inequality. The
first term in (\ref{eq:two-terms}) can be bounded as below:
\begin{eqnarray*}
	&  & \e\left[\frac{1}{E_{S}}\middle\vert E_{S}>0,\left|E_{S}-\binom{n}{2}p\right|\leq\delta\right]\times\\
	&  & \p\left[\left|E_{S}-\binom{n}{2}p\right|\leq\delta\given E_{S}>0\right]\\
	& \leq & \e\left[\frac{1}{E_{S}}\middle\vert E_{S}>0,\left|E_{S}-\binom{n}{2}p\right|\leq\delta\right]\\
	& \leq & \frac{1}{\binom{n}{2}p-\delta}.
\end{eqnarray*}
Then we combine the two upper bounds together and obtain
\[\e\left[\frac{1}{E_{S}}\middle\vert E_{S}>0\right]\leq 
\frac{1}{\binom{n}{2}p-\delta}+\frac{2e^{-2\delta^{2}/\binom{n}{2}}}{1-\p\left[E_{S}=0\right]}.
\]
Let $\delta=n^{3/2}$, we have
\begin{eqnarray*}
	&  & \e\left[\hat{N}_{NS}|E_{S}>0\right]\\
	& \leq & n\left(\frac{n(N-n)p}{2}\left(\frac{1}{\binom{n}{2}p-\delta}+\frac{2e^{-2\delta^{2}/\binom{n}{2}}}{1-\p\left[E_{S}=0\right]}\right)+1\right)\\
	& = & n\left(\frac{n(N-n)p}{2}\left(\frac{1}{\binom{n}{2}p-n^{3/2}}+\frac{2e^{-2n^{3}/\binom{n}{2}}}{1-\p\left[E_{S}=0\right]}\right)+1\right)\\
	& \to & N,
\end{eqnarray*}
as $n$ goes to $\infty$. 
Recall that we have shown that $$\e\left[\hat{N}_{NS}|E_{S}>0\right]>N.$$
Thus 
$
\lim_{n\to\infty}\e\left[\hat{N}_{NS}|E_{S}>0\right]=N.
$
Hence we conclude that $\hat{N}_{NS}$ is asymptotically unbiased.

\subsection*{Proof of Theorem \ref{thm:regularity}}
We notice that 
$
\sum_{y}L(\tn,y;\d,p)=L(\tn;\d,p).
$
This is a finite sum since $\tilde{d}(v)$ is a known and fixed integer
and thus the total number of possibilities of integer composition
of $\tilde{d}(v)$ is finite. Therefore, it suffices to show the regularity
results for the joint posterior $\int L(\tn,y;\d,p)\pi(\tn,p)dp$
for every fixed $y$. Hereinafter, we fix an arbitrary $y$. We have
\begin{align*}
	& L_{\neg W}(\d;p,y) = \prod_{v\in W}\prod_{i=1}^{K}\binom{\tn_{i}}{y_{i}(v)}p_{i,t(v)}^{y_{i}(v)}(1-p_{i,t(v)})^{\tn_{i}-y_{i}(v)}\\
	= & \zeta(\tn)\prod_{j=1}^{K}\prod_{i=1}^{K}\prod_{v\in W,t(v)=j}p_{ij}^{y_{i}(v)}(1-p_{ij})^{\tn_{i}-y_{i}(v)}\\
	= & \zeta(\tn)\prod_{j=1}^{K}\prod_{i=1}^{K}p_{ij}^{S_{ji}}(1-p_{ij})^{M_{ji}}\\
	= & \zeta(\tn)\prod_{1\leq i<j\leq K}p_{ij}^{S_{[ij]}}(1-p_{ij})^{M_{[ij]}}\prod_{i=1}^{K}p_{ii}^{S_{ii}}(1-p_{ii})^{M_{ii}},
\end{align*}
where $\zeta(\tn)=\prod_{v\in W}\prod_{i=1}^{K}\binom{\tn_{i}}{y_{i}(v)}$,
$S_{ji}=\sum_{v\in W,t(v)=j}y_{i}(v)$, $M_{ji}=\sum_{v\in W,t(v)=j}(\tn_{i}-y_{i}(v))=\tn_{i}V_{j}-S_{ji}$,
$S_{[ij]}=S_{ij}+S_{ji}$ and $M_{[ij]}=M_{ij}+M_{ji}$. 
Furthermore, we have
\begin{align*}
	& L_{W}(\d;p)L_{\neg W}(\d;p,y)\psi(p)/\zeta(\tn)\\
	\propto & \left(\prod_{1\leq i<j\leq K}p_{ij}^{E_{ij}+S_{[ij]}+\alpha_{ij}}(1-p_{ij})^{V_{i}V_{j}-E_{ij}+M_{[ij]}+\beta_{ij}}\right)\times\\
	& \left(\prod_{i=1}^{K}p_{ii}^{E_{ii}+S_{ii}+\alpha_{ii}}(1-p_{ii})^{\binom{V_{i}}{2}-E_{ii}+M_{ii}+\beta_{ii}}\right)
\end{align*}
Therefore
\begin{align*}
	& \p(\tn,y|\d) \propto \int L(\tn,y;\d,p)\pi(\tn,p)dp\\
	= & \int L_{W}(\d;p)L_{\neg W}(\d;p,y)\phi(\tn)\psi(p)dp\\
	= & \zeta(\tn)\phi(\tn)\int L_{W}(\d;p)L_{\neg W}(\d;p,y)\psi(p)/\zeta(\tn)dp\\
	\propto & \zeta(\tn)\phi(\tn)\int\left(\prod_{i=1}^{K}p_{ii}^{E_{ii}+S_{ii}+\alpha_{ii}}(1-p_{ii})^{\binom{V_{i}}{2}-E_{ii}+M_{ii}+\beta_{ii}}\right)\times\\
	& \left(\prod_{1\leq i<j\leq K}p_{ij}^{E_{ij}+S_{[ij]}+\alpha_{ij}}(1-p_{ij})^{V_{i}V_{j}-E_{ij}+M_{[ij]}+\beta_{ij}}\right)dp,\\
	= & \zeta(\tn)\phi(\tn)\prod_{i=1}^{K}\b\left(E_{ii}+S_{ii}+\alpha_{ii},\binom{V_{i}}{2}-E_{ii}+M_{ii}+\beta_{ii}\right)\times\\
	& \prod_{1\leq i<j\leq K}\b\left(E_{ij}+S_{[ij]}+\alpha_{ij},V_{i}V_{j}-E_{ij}+M_{[ij]}+\beta_{ij}\right).
\end{align*}
We notice that $M_{[ij]}=M_{ij}+M_{ji}=\tn_{i}V_{j}-S_{ji}+\tn_{j}V_{i}-S_{ij}=\tn_{i}V_{j}+\tn_{j}V_{i}-S_{[ij]}$.
Let $\eta_{ii}=E_{ii}+S_{ii}+\alpha_{ii}$ and $\eta_{ij}=E_{ij}+S_{[ij]}+\alpha_{ij}$
for $i\neq j$ and let $\theta_{ii}=\binom{V_{i}}{2}-E_{ii}-S_{ii}+\beta_{ii}$
and $\theta_{ij}=V_{i}V_{j}-E_{ij}-S_{[ij]}+\beta_{ij}$ for $i<j$.
Thus we have 
\[
\p(\tn,y|\d)  \propto  \zeta(\tn)\phi(\tn)\prod_{i=1}^{K}\b(\eta_{ii},\theta_{ii}+\tn_{i}V_{i})\times
\prod_{1\leq i<j\leq K}\b(\eta_{ij},\theta_{ij}+\tn_{i}V_{j}+\tn_{j}V_{i});
\]
i.e., \[\p(\tn,y|\d)=C\zeta(\tn)\phi(\tn)\prod_{i=1}^{K}\b(\eta_{ii},\theta_{ii}+\tn_{i}V_{i})\times
\prod_{1\leq i<j\leq K}\b(\eta_{ij},\theta_{ij}+\tn_{i}V_{j}+\tn_{j}V_{i})
\]
for some fixed constant $C$.

In \cite{wendel1948note}, Wendel proved 
$
\lim_{y\to+\infty}\frac{\Gamma(x+y)}{y^{x}\Gamma(y)}=1,
$
for real $x$ and $y$. Therefore, there exists a number $\nu(x)$
such that $\forall y>\nu(x)$, we have $\frac{2\Gamma(x+y)}{y^{x}\Gamma(y)}\geq1;$
i.e., 
$
2y^{-x}\geq\frac{\Gamma(y)}{\Gamma(x+y)}.
$
This is equivalent to
$
2\Gamma(x)y^{-x}\geq\b(x,y).
$
Suppose that $\tn_{i}$'s are sufficiently large, i.e., $\theta_{ii}+\tn_{i}V_{i}>\nu(\eta_{ii})$
for all $i$ and $\theta_{ij}+\tn_{i}V_{j}+\tn_{j}V_{i}>\nu(\eta_{ij})$
for all $i<j$. In this case, we have
\begin{align*}
	& \p(\tn,y|\d) \leq C_{1}\zeta(\tn)\phi(\tn)\prod_{i=1}^{K}\Gamma(\eta_{ii})(\theta_{ii}+\tn_{i}V_{i})^{-\eta_{ii}}\times\\
	& \prod_{1\leq i<j\leq K}\Gamma(\eta_{ij})(\theta_{ij}+\tn_{i}V_{j}+\tn_{j}V_{i})^{-\eta_{ij}}\\
	= & C_{2}\zeta(\tn)\phi(\tn)\prod_{i=1}^{K}(\theta_{ii}+\tn_{i}V_{i})^{-\eta_{ii}}\times \\
	& \prod_{1\leq i<j\leq K}(\theta_{ij}+\tn_{i}V_{j}+\tn_{j}V_{i})^{-\eta_{ij}}\\
	\leq & C_{3}\zeta(\tn)\phi(\tn)\prod_{i=1}^{K}(\theta_{ii}+\tn_{i}V_{i})^{-\eta_{ii}}\times\\
	& \prod_{1\leq i<j\leq K}(\theta_{ij}+2\tn_{i}V_{j})^{-\eta_{ij}}(\theta_{ij}+2\tn_{j}V_{i})^{-\eta_{ij}}\\
	\leq & C_{3}\zeta(\tn)\phi(\tn)\prod_{i=1}^{K}(\theta_{min}+\tn_{i}V_{min})^{-\eta_{ii}}\times\\
	& \prod_{1\leq i<j\leq K}(\theta_{min}+\tn_{i}V_{min})^{-\eta_{ij}}(\theta_{min}+2\tn_{j}V_{min})^{-\eta_{ij}}\\
	\leq & C_{4}\zeta(\tn)\phi(\tn)\prod_{i=1}^{K}\tn_{i}^{-\eta_{ii}}\prod_{1\leq i<j\leq K}\tn_{i}^{-\eta_{ij}}\tn_{j}^{-\eta_{ij}},
\end{align*}
where $C_{1}=2^{K+\binom{K}{2}}C$, $C_{2}=C_{1}\prod_{i=1}^{K}\Gamma(\eta_{ii})\prod_{1\leq i<j\leq K}\Gamma(\eta_{ij})$,
$C_{3}=C_{2}/2^{\binom{K}{2}}$, $\theta_{min}=\min\{\theta_{ij}:1\leq i\leq j\leq K\}$,
$V_{min}=\min\{V_{ij}:1\leq i\leq j\leq K\}$, and the last inequality
holds for some fixed constant $C_{4}$ and sufficiently large $\tn_{i}$'s.
Now consider the term $\zeta(\tn)$. We have
\[
\zeta(\tn)=\prod_{v\in W}\prod_{i=1}^{K}\binom{\tn_{i}}{y_{i}(v)}\leq\prod_{v\in W}\prod_{i=1}^{K}\tn_{i}^{y_{i}(v)}
=\prod_{i=1}^{N}\tn_{i}^{\sum_{v\in W}y_{i}(v)}=\prod_{i=1}^{N}\tn_{i}^{\sum_{j=1}^{K}S_{ji}}.
\]
Therefore,
\begin{align*}
	&  \p(\tn,y|\d) \leq C_{4}\zeta(\tn)\phi(\tn)\prod_{i=1}^{K}\tn_{i}^{-\eta_{ii}}\prod_{i\neq j}\tn_{i}^{-\eta_{ij}/2}\tn_{j}^{-\eta_{ij}/2}\\
	= & C_{4}\zeta(\tn)\phi(\tn)\prod_{i=1}^{K}\tn_{i}^{-\eta_{ii}}\prod_{i=1}^{K}\tn_{i}^{-\sum_{j\neq i}\eta_{ij}}\\
	= & C_{4}\phi(\tn)\prod_{i=1}^{K}\tn_{i}^{\sum_{j=1}^{K}(S_{ji}-\eta_{ji})}\\
	= & C_{4}\phi(\tn)\prod_{i=1}^{K}\tn_{i}^{-\sum_{j=1}^{K}(E_{ij}+S_{ij}+\alpha_{ij})}\\
	\leq & C_{4}\phi(\tn)\prod_{i=1}^{K}\tn_{i}^{-\sum_{j=1}^{K}(E_{ij}+\alpha_{ij})}.
\end{align*}
Now we consider the $n$-th moment of $N$. Let $$\mu_{n}=N^{n}\p(\tn,y|\d).$$
We have
\begin{eqnarray*}
	\mu_{n} & \leq & C_{4}\phi(\tn)N^{n}\prod_{i=1}^{K}\tn_{i}^{-\sum_{j=1}^{K}(E_{ij}+\alpha_{ij})}\\
	& \leq & C_{5}\phi(\tn)\sum_{i=1}^{K}(\tn_{i}+|W|/K)^{n}\prod_{i=1}^{K}\tn_{i}^{-\sum_{j=1}^{K}(E_{ij}+\alpha_{ij})}
\end{eqnarray*}
where $C_{5}=C_{4}K^{n-1}$. There exists a constant $C_{6}$ such
that for sufficiently large $\tn_{i}$'s, we have $\tn_{i}+|W|/K\leq C_{6}\tn_{i}$
for $\forall1\leq i\leq K$. Therefore, 
\begin{eqnarray*}
	\mu_{n} & \leq & C_{5}C_{6}^{n}\phi(\tn)\sum_{i=1}^{K}\tn_{i}^{n}\prod_{i=1}^{K}\tn_{i}^{-\sum_{j=1}^{K}(E_{ij}+\alpha_{ij})}\\
	& = & C_{7}\phi(\tn)\sum_{i=1}^{K}\tn_{i}^{n}\prod_{i=1}^{K}\tn_{i}^{-\sum_{j=1}^{K}(E_{ij}+\alpha_{ij})}\\
	& \leq & KC_{7}\phi(\tn)\prod_{i=1}^{K}\tn_{i}^{-(\sum_{j=1}^{K}(E_{ij}+\alpha_{ij})-n)}\\
	& \leq & C_{8}\phi(\tn)(\prod_{i=1}^{K}\tn_{i})^{-(\lambda-n)},
\end{eqnarray*}
where $C_{7}=C_{5}C_{6}^{n}$ and $C_{8}=KC_{7}$. Since $\phi(\tn)$
is bounded, $\sum_{\tn}\mu_{n}<\infty$ if $\lambda>n+1$.

\section*{Proposal for $y(v)$ and Its Proposal Ratio}


The proposal algorithm for $y(v)$ is presented in Algorithm~\ref{alg:Proposal-algorithm-for-y}.

\begin{algorithm}[hbt]
	\algsetup{linenosize=\small}
	\small
	\begin{multicols}{2}
		\begin{algorithmic}[1]
			
			\ENSURE The state in the previous round, $y(v)^{(\t-1)}$
			
			\REQUIRE The proposed new state, $y(v)^{*}$
			
			\STATE $y(v)^{*}\gets y(v)^{(\t-1)}$
			
			\LOOP
			
			\STATE Sample two distinct $i,j\in[1,K]\cap\mathbb{N}$ uniformly
			at random
			
			\IF{$y_{i}(v)^{*}>0$ \AND $y_{j}(v)^{*}<\tn_{j}$}
			
			\STATE Exit loop
			
			\ENDIF
			
			\ENDLOOP
			
			\STATE $y_{i}(v)^{*}\gets y_{i}(v)^{*}-1$
			
			\STATE $y_{j}(v)^{*}\gets y_{j}(v)^{*}+1$%
		\end{algorithmic}
	\end{multicols}
	
	%
	%
	%
	%
	%
	%
	%
	%
	
	\caption{Proposal algorithm for $y(v)$\label{alg:Proposal-algorithm-for-y}}
\end{algorithm}

Theorem~\ref{thm:proposal-ratio} shows the proposal ratio of Algorithm \ref{alg:Proposal-algorithm-for-y}.
\begin{thm}
	\label{thm:proposal-ratio}Let
	$
	\av(y(v))=\sum_{i=1}^{K}\sum_{j=1}^{K}1\{i\neq j,y_{i}(v)>0,y_{j}(v)<\tn_{j}\}.
	$
	The proposal ratio of Algorithm \ref{alg:Proposal-algorithm-for-y}
	is given by
	$
	\frac{h_{v}(y(v)^{*}\to y(v)^{(\t-1)})}{h_{v}(y(v)^{(\t-1)}\to y(v)^{*})}  =  \frac{1/\av(y(v)^{*})}{1/\av(y(v)^{(\t-1)})}
	=  \frac{\av(y(v)^{(\t-1)})}{\av(y(v)^{*})}.
	$
\end{thm}

\begin{proof}
	
	Let $A=\{(i,j)\in([1,K]\cap\mathbb{N})^{2}:i\neq j,y_{i}(v)>0,y_{j}(v)<\tn_{j}\}$
	be the space of all pairs of $i$ and $j$ that satisfy the condition
	of exiting the loop in Algorithm \ref{alg:Proposal-algorithm-for-y}.
	And each pair of $(i,j)$ in $A$ correspond to a new state with $y_{i}(v)$
	subtracted by $1$ and $y_{j}(v)$ added by $1$. The size of $A$
	is given by $\av(y(v))$. The proposal algorithm proposes each new
	state uniformly at random. Thus 
	$
	h_{v}(y(v)^{*}\to y(v)^{(\t-1)})=\frac{1}{\av(y(v)^{*})}.
	$
	Similarly,
	$
	h_{v}(y(v)^{(\t-1)}\to y(v)^{*})=\frac{1}{\av(y(v)^{(\t-1)})}.
	$
	This completes the proof.
\end{proof}

\end{document}